\documentclass{article}

\usepackage{microtype}
\usepackage{graphicx}
\usepackage{subcaption}
\usepackage{booktabs} 

\usepackage{hyperref}
\usepackage{nicefrac}

 \usepackage[preprint]{icml2026}

\usepackage{amsmath}
\usepackage{amssymb}
\usepackage{mathtools}
\usepackage{amsthm}
\usepackage[normalem]{ulem}
\usepackage{multirow}
\usepackage{enumitem}
\usepackage{fvextra}
\usepackage[capitalize,noabbrev]{cleveref}

\theoremstyle{plain}
\newtheorem{theorem}{Theorem}[section]
\newtheorem{proposition}[theorem]{Proposition}

\theoremstyle{definition}

\newtheorem{assumption}[theorem]{Assumption}
\theoremstyle{remark}

\usepackage[textsize=tiny]{todonotes}

\usepackage{xcolor}

\icmltitlerunning{Self-Generative Adversarial Fine-Tuning for Large Language Models}

\begin{document}

\twocolumn[
\icmltitle{Self-Generative Adversarial Fine-Tuning for Large Language Models}



\icmlsetsymbol{equal}{*}

\begin{icmlauthorlist}
	\icmlauthor{Shiguang Wu}{thu}
	\icmlauthor{Yaqing Wang}{bimsa}
	\icmlauthor{Quanming Yao}{thu}

\end{icmlauthorlist}

\icmlaffiliation{thu}{Department of Electronic Engineering, Tsinghua University}
\icmlaffiliation{bimsa}{Beijing Institute of Mathematical Sciences and Applications}

\icmlcorrespondingauthor{Quanming Yao}{qyaoaa@tsinghua.edu.cn}

\icmlkeywords{Machine Learning, ICML}

\vskip 0.3in
]



\printAffiliationsAndNotice{}  

\begin{abstract}

Fine-tuning large language models (LLMs) for alignment typically relies on supervised fine-tuning or reinforcement learning from human feedback, both limited by the cost and scarcity of high-quality annotations. Recent self-play and synthetic data approaches reduce this dependence but often rely on heuristic assumptions or ungrounded self-evaluation, which can cause bias accumulation and performance drift. In this paper, we propose Self-Generative Adversarial LLM (SGALM), a unified fine-tuning framework that formulates alignment as a generative adversarial game within a single LLM. SGALM jointly evolves generation and discrimination capabilities without external reward models. Theoretical and empirical results 
demonstrate that SGALM achieves state-of-the-art performance, serves as an effective alignment algorithm and a robust synthetic data engine.
\end{abstract}

\section{Introduction}

Large Language Models (LLMs) have emerged as a paradigm shift in artificial intelligence, exhibiting remarkable general intelligence across diverse capabilities. While pre-training on massive corpora equips these models with a vast reservoir of world knowledge and general instruction-following capacity, this ``jack-of-all-trades'' foundation is often insufficient for domain-specific excellence or alignment with human preferences \cite{roziere2023code,ouyang2022training}. These general models may lack precise adherence to instruction formats, safety protocols, or specialized terminology required for downstream applications. Consequently, adapting LLMs serves as the critical bridge that aligns a model’s broad, generic capabilities with task specialization or human preference.

The standard pipeline for adapting LLMs typically starts with Supervised Fine-Tuning (SFT), often followed by Reinforcement Learning from Human Feedback (RLHF) \cite{bai2022training} or Direct Preference Optimization (DPO) \cite{rafailov2023direct}. While successful, these methods are fundamentally constrained by the scarcity and cost of high-quality human annotations. To address this bottleneck, the field has increasingly turned towards utilizing LLMs to generate their own training signals. This includes generating synthetic data for SFT~\cite{wang2023self,honovich2023unnatural} and more sophisticated self-play mechanisms \cite{gulcehre2023reinforced,xiong2024iterative,chen2024self,yuan2024self,pang2024iterative}, where the model evolves as both the student and the teacher in a closed feedback loop.

However, these self-play methods exhibit structural limitations in how they distinguish high-quality synthetic data from hallucinations. SPIN \cite{chen2024self}, while effective, operates by minimizing the likelihood of model-generated responses while maximizing real responses. This creates a strict dichotomy that implicitly assumes all model generations are inferior to the ground truth. In open-ended reasoning tasks, a model-generated solution may be factually correct but phrased differently than the reference; SPIN will penalize this valid variation. Self-rewarding models \cite{yuan2024self} rely on the LLM’s zero-shot capability to act as a judge. Without continuous grounding in ground-truth data, this ``judge'' can suffer from reward hacking or drift, reinforcing the model’s own biases rather than the true data distribution.

Inspired by the architectural success of Generative Adversarial Networks (GANs) \cite{goodfellow2014generative}, we argue that an LLM's ability to generate realistic data can be improved by attempting to fool a discriminator, while the discriminator is improved by learning to distinguish between real and generated data. Considering the duality of generation and discrimination capacities in general intelligence:
\begin{quote}
	\centering \itshape
	``What I cannot create, I do not understand.'' \\
	--- Richard Feynman
\end{quote}
we propose that generation and discrimination capacities should be evolved as dual capabilities within an LLM. The intuition is that if an artificial intelligence can generate realistic (diverse and correct) samples in a domain, then it becomes an expert in that domain.

We address the non-trivial challenge of applying the GAN framework to discrete text generation and discrimination within an LLM by utilizing in-context learning (ICL) for generation and the output distribution for optimizable discrimination. Furthermore, we generalize the few-shot generation capacity utilized during GAN training to a ready-to-use zero-shot understanding capacity, theoretically grounded by the Bayesian nature of ICL. Thanks to the human-like intelligence and the powerful, consistent input-output spaces (natural language) of contemporary LLMs, a single LLM can serve as both the generator and discriminator in a GAN architecture.

This results in the Self-Generative Adversarial Fine-Tuning LLM (SGALM). Requiring only a pre-trained LLM and a real dataset to align, SGALM plays a GAN-like minimax game with itself, as illustrated in Figure~\ref{fig:illus}, and simultaneously serves two functions:
\textbf{Fine-tuning for alignment}: The resulting model possesses aligned capabilities ready for use, as generation ability overlaps with understanding ability in general intelligence; and
\textbf{Creating a synthetic data engine}: The resulting model can generate and filter to obtain more high-fidelity data from the distribution of the training set, which can be used for further fine-tuning.

We summarize our contributions as follows:
\begin{itemize}[leftmargin=*]
	\item We introduce SGALM, which formulates LLM fine-tuning as a self-contained adversarial game where a single LLM jointly evolves its generation and discrimination capabilities. This reframes self-play fine-tuning from static preference optimization into a principled GAN-style alignment process without external dependency.
	
	\item SGALM realizes the adversarial game using a single shared-parameter LLM, leveraging ICL for diverse and flexible generation and continuous ``Real/Fake'' judgments for discrimination. We provide a theoretical analysis showing that the resulting alternating updates recover the true data distribution at equilibrium, where the few-shot generation capacity also leads to zero-shot understanding.
	
	\item Extensive experiments on GSM8K, ARC-Challenge, and MBPP demonstrate that SGALM consistently outperforms supervised fine-tuning and prior self-play baselines. Moreover, SGALM uniquely exhibits positive scaling behavior as the volume of synthetic data increases, validating its effectiveness as a high-fidelity synthetic data engine that mitigates overfitting and model collapse.
\end{itemize}

\section{Preliminaries}
\label{sec:set}

To learn the true data distribution $p_T$, GAN framework employs two models: 
a generator $G$ defining a distribution $p_G$, and a discriminator $D$ outputting a scalar, engaging in a minimax game:
\begin{align}
\label{eq:gan}
\!\!\!\!\!
\min_{G} \! \max_{D} 
\! J \! \equiv \! 
\mathbb{E}_{z'\sim p_G}[\log(1 \! - \! D(z'))] 
\! + \! \mathbb{E}_{z\sim p_T} \! [\log D(z)].
\!\!\!
\end{align}
The discriminator $D$ is trained to distinguish between real samples $z$ drawn from the data distribution $p_T$
and fake samples $z'$ produced by the generator $G$. 
Simultaneously, the generator $G$ aims to synthesize realistic samples that deceive the discriminator $D$, minimizing the $D$'s ability to differentiate them from real data. 

The minimax game is played by iteratively updating $D$ and $p_G$: in every iteration, $z'$ and $z$ are drawn from $p_G$ and $p_T$ respectively; first $D$ is updated towards $\max_{D} J$; then $G$ is updated towards $\min_{G} J$.
It theoretically drives the generator $p_G$ to recover the true data distribution $p_T$
when the discriminator can no longer distinguish generated samples from real ones.
Due to the complexity of the bi-level optimization problem 
\eqref{eq:gan}, 
the training of GANs typically makes approximation by decompose \eqref{eq:gan} into a discrimination objective $\max_{D} J(D)$, 
and a generation objective $\min_{G} J(G)$, and optimizes them iteratively.

 In computer vision, the adversarial game has been successfully deployed for diverse applications such as domain adaptation \cite{ganin2015unsupervised}, style transfer \cite{zhu2017unpaired}, and data augmentation \cite{antoniou2017data}. In natural language processing, however, prior attempts \cite{yu2017seqgan,che2017maximum,guo2018long} primarily focused on employing Reinforcement Learning to bypass the non-differentiability of discrete tokens and were limited to training small models for text generation. 
 LLMs have not yet witnessed a general and complete generative adversarial game; the proposed SGALM achieves this for fine-tuning for alignment and creating a synthetic data engine.
 Broader related works are discussed in Appendix~\ref{app:rw}.

\begin{figure*}[t]
\centering
\includegraphics[width=0.93\textwidth]{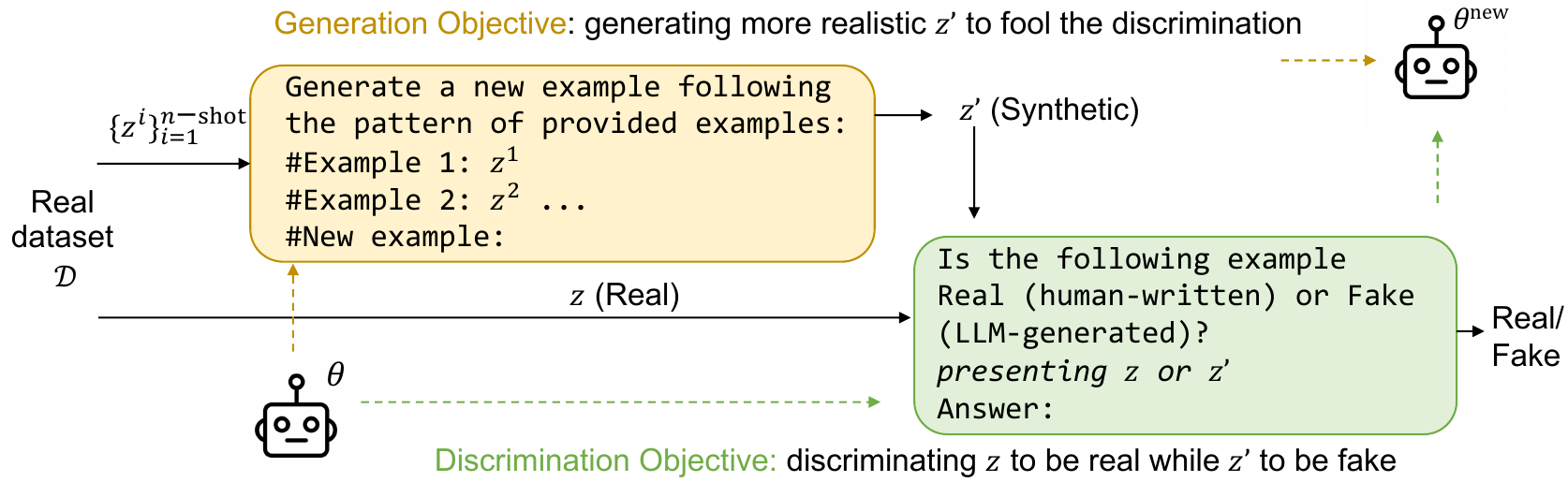}
\caption{Illustration of the proposed self-generative adversarial fine-tuning framework.}
\vspace{-10px}
\label{fig:illus}
\end{figure*}

\section{Our Method}

As natural language data, each sample $z$ is represented as a sequence, where $|z|$ denotes its length and $z_t$ denotes the $t$-th element. Any sequence distribution $p(z)$ follows an auto-regressive factorization  $p(z)=\prod_{t=1}^{|z|} p(z_t \mid z_{<t})$. We consider a LLM parameterized by $\theta$, which maps a context prefix $z_{<t}$ to an output distribution $p_\theta(\cdot \mid z_{<t})$. Given a dataset $\mathcal{D}=\{ z \mid z \sim p_T(z) \}$ drawn from a target distribution $p_T$ corresponding to domain $T$, our objective is to align the model distribution $p_\theta$ with $p_T$ through fine-tuning. Depending on the domain, a sample $z$ may take various textual forms, including standalone documents, instruction-response pairs, or decision trajectories.

The underlying intuition is that an intelligent system capable of generating realistic, diverse, and correct samples within a domain effectively demonstrates expertise in that domain. Moreover, due to the Bayesian nature of ICL in LLMs, such a generator can naturally function as a ready-to-use expert assistant without additional prompting. Motivated by this insight, we adopt a generative adversarial perspective to guide model alignment.

Specifically, given a dataset $\mathcal{D}$ and a generative model $p_\theta$, we design the learning process around three core components: (i) \textbf{Generation}, where synthetic samples $z'$ are produced from $p_\theta$; (ii) \textbf{Discrimination}, where a scalar score is assigned to an input $z$ or $z'$; and (iii) \textbf{Update}, where the model parameters $\theta$ are optimized accordingly. However, instantiating this framework for LLMs is non-trivial. 
Challenges are ensuring sufficient diversity in generated samples, designing a discrimination signal $D(z)$ that is differentiable and optimizable with respect to $\theta$, and decomposing the adversarial objective and performing stable updates under the constraints of discrete text generation and a unified generator--discriminator architecture.

\subsection{Generation}
To function effectively as a generator within the GAN framework, the generation of $z'$ is expected to be diverse enough, as drawing from a continuous distribution. Typical GANs implement $G$ using a model that takes noise drawn from a continuous distribution as input, but LLMs lacks such mechanism. To learn the target distribution $p_T$ while preserving general intelligence, we utilize the few-shot ICL capabilities of modern LLMs \cite{brown2020language,wang2023self}.

We sample a few-shot subset $\{z^i\}_{i=1}^{n\text{-shot}} \subset \mathcal{D}$ to construct a context prompt $Z_{\text{ctx}}$ containing the $n$ real examples. The model generates a "fake" sample $z'$ via:
\begin{align}\label{eq:gen}
	z' \sim p^G_\theta(z')\equiv p_{\theta}(z' \mid G_\text{Prompt}, Z_{\text{ctx}}),
\end{align}
where $G_\text{Prompt}$ is the generative instruction in the prompt, such as ``\textit{generate a new example following the pattern of provided examples}''. Since LLMs possess ICL capacity to identify patterns in context, it is expected that the LLM can learn $p_T$ from the provided $Z_{\text{ctx}}$. Meanwhile, as the number of combinations and permutations of examples $\{z^i\}$ from $\mathcal{D}$ to form $Z_{\text{ctx}}$ is vast, and the inherent randomness in the generation process $p_\theta(\cdot \mid x)$ is controllable, the generation of $z'$ maintains sufficient diversity.

\subsection{Discrimination}
The discrimination mechanism is expected to output a scalar $D(z)$ for each sample $z$ (or $z'$), to evaluate the objective \eqref{eq:gan}.
Typical GANs implement the discriminator model with a classifier with demanded output. For using LLM as discriminator, there are implementations of using additional module to map the final hidden states \cite{yu2023fine}, or directly asking the LLM to output a scalar score in text \cite{yuan2024self}. However, they are not feasible in SGALM, as the former one breaks the behavior of LLM as general intelligence, while the latter one requires first fine-tuning on scoring dataset.

In SGALM, we ask the LLM if the input is ``Real" or ``Fake", and obtain the scalar score based on its output distribution. Specifically, given an input $z$, we have a distribution 
$p_{\theta}(\cdot | D_\text{Prompt}, z),$
where $D_\text{Prompt}$ is the discriminative instruction in prompt, like
``\textit{Is the example Real (human-written) or Fake (LLM-generated)? Answer with one word}".
Then we have the output probability $p_{\theta}(\text{Real} | D_\text{Prompt}, z)$ and $p_{\theta}(\text{Fake} | D_\text{Prompt}, z)$. Define the normalized binary distribution on $\{\text{``Real''} , \text{``Fake''}  | z\}$ as
\begin{align}\label{eq:dis2}
\!\!\!\!
p^\text{real}_\theta(z)
\! \equiv \!
\frac{p_{\theta}(\text{Real} | D_\text{Prompt}, z)}
{p_{\theta}(\text{Real} | D_\text{Prompt}, z)  
\! + \! 
p_{\theta}(\text{Fake} | D_\text{Prompt}, z)},
\end{align}
and $p^\text{fake}_\theta(z)\equiv1-p^\text{real}_\theta(z)$.
They are continuous value implying how real/fake the LLM think the sample $z$ is. Thus a $D(z)=p^\text{real}_\theta(z)$ (or any continuous function in $(0,1)$ monotonic increasing with $p^\text{real}_\theta(z)$) can be used for effectively optimizing both the generation and discrimination.

\subsection{Updating}

In SGALM, 
the discrimination objective is straight-forward, maximizing the probability of telling ``Real'' for real samples $z$ while telling ``Fake'' for fake samples $z'$.
Let  $D(z)$ be some continuous function in $(0,1)$ monotonic increasing with $p^{real}_\theta(z)$, the discrimination objective is
\begin{align}
\nonumber
& \max_D \mathbb{E}_{z'\sim p_G}[\log(1-D(z'))]+\mathbb{E}_{z\sim p_T}[\log D(z)] = 
\\
& \min_{\theta} \! -\mathbb{E}_{z'\sim p_G}[\log(1-D_\theta(z'))] \! - \! \mathbb{E}_{z\sim p_T}[\log D_\theta(z)].
\label{eq:d1}
\end{align}
Note that specifically in SGALM, $G$ and $D$ are implemented with the same model $\theta$, i.e., $p_G(z')=p_\theta^G(z')$ and $D(z)=D_\theta(z)$. However, the objective here is to discriminate real/fake samples better, rather than generating worse. So denoting a detached copy of $\theta$ as $\theta^\dag$, 
\eqref{eq:d1}
\begin{align*}
=\min_{\theta} 
\! -
\mathbb{E}_{z'\sim p_{\theta^\dag}^G}[\log(1 \! - \! D_\theta(z'))]
\! - \!
\mathbb{E}_{z\sim p_T}[\log D_\theta(z)].
\end{align*}
The discriminator parameters are updated by taking gradient steps with respect to the discriminator objective, by 
\begin{align}
\nabla_\theta J(D)
=-\int p_{\theta}^G(z') \nabla_\theta \log(1 \! - \! D_\theta(z')) dz'
\notag\\-
\int p_T(z) \nabla_\theta \log D_\theta(z) dz. \label{eq:gd-d}
\end{align}
The expectations are approximated via Monte Carlo sampling using
fake samples generated from $p_{\theta}^G$ and real samples from the dataset $\mathcal{D}$.

And the generation objective is to generate better samples to fool the discrimination,
\begin{align}
\!\!\!\!
\min_{G} \mathbb{E}_{z' \! \sim \! p_G}[\log(1 \!\! - \!\! D(z'))]
\! = \!
\min_{\theta} \mathbb{E}_{z' \! \sim \! p_{\theta}^G}[\log(1 \!\! - \!\! D(z')]. 
\!\!
\label{eq:g1}
\end{align}
Similarly, the objective here is to generate more realistic samples rather than discriminate worse. So \eqref{eq:g1}
\begin{align*}
	=\min_{\theta} \mathbb{E}_{z' \sim p_{\theta}^G}[\log(1-D_{\theta^\dag}(z'))],
\end{align*}
which can be done by taking steps with generation gradient
\begin{align}
\!\!\!\!
\nabla_\theta J(G)
\! = \!\!
\int 
\!\!
p_{\theta}^G(z')
\log(1 \! - \! D_{\theta}(z'))
\nabla_\theta \log p_\theta^G(z')
dz'.
\label{eq:gd-g}
\end{align}
This can also be estimated by Monte-Carlo sampling from fake dataset generated from $p_{\theta}^G$.
The complete algorithm is provided in Algorithm~\ref{alg:sgalm}.

\begin{algorithm}[ht]
	\caption{SGALM Algorithm.}
	\label{alg:sgalm}
	\small
	\begin{algorithmic}
		\STATE {\bfseries Input:} Dataset $\mathcal{D}$, LLM $\theta$, iteration number $t$.
		\FOR{$iter = 0, \dots, t-1$}
		\STATE Generate fake dataset $\mathcal{D}'=\{z'\}$ by \eqref{eq:gen}.
		\STATE Discrimination update for \eqref{eq:d1} with \eqref{eq:gd-d}, and generation update for \eqref{eq:g1} with \eqref{eq:gd-g}.
		
		%
		\ENDFOR
	\end{algorithmic}
\end{algorithm}

\section{Theoretical Analysis}

We provide theoretical analysis showing that 
(i) the generative adversarial game converges when the generation distribution is exactly aligned to the true distribution, and (ii) though SGALM is trained with few-shot generation, it is also capable of zero-shot understanding of target domain.

\begin{table*}[ht]\centering\small
	\caption{Comparison with existing works. SGALM is unique in establishing a self-contained generative adversarial loop without relying on external reward models, ground truth, nor specific assumption on the existing LLM's generation/discrimination quality.}
	\label{tab:comparison_works}
	\resizebox{\textwidth}{!}{
		\begin{tabular}{lccc}
			\toprule
			\textbf{Method} & \textbf{Reward Dependency/Assumption} & \textbf{Generation Scope} & \textbf{Data Filter / Supervision}\\
			\midrule
			Self-Instruct \cite{wang2023self} & None & New $(x,y)$ & Defined Rules  \\
			ReST \cite{gulcehre2023reinforced} & External Reward Model & Given $x$, New $y$& Ext. Reward \\
			Iterative-DPO \cite{xiong2024iterative} & External Reward Model & Given $x$, New $y$& Win/Lose Pair via Ext. Reward  \\
			Iterative-RPO \cite{pang2024iterative} & Ground Truth Answer & Given $x$, New $y$& Win/Lose Pair via True/False \\
			\midrule
			Self-Rewarding \cite{yuan2024self} & LLM excels in scoring& New $(x,y)$  &Score by  LLM-as-Judge \\
			SPIN \cite{chen2024self} &  LLM generation is always bad & Given $x$, New $y$& Real/Generated Pair \\
			\midrule
			\textbf{SGALM (Ours)} & \textbf{None} & \textbf{New $z$, including $(x,y)$} & \textbf{Adversarial  Discrimination} \\
			\bottomrule
		\end{tabular}
	}
\vspace{-10px}
\end{table*}
\subsection{Distribution Alignment via Generative Adversarial Game}
We first analyze the behavior of the discrimination update, i.e., updating \eqref{eq:d1} with \eqref{eq:gd-d}. This process is mathematically equivalent to a binary classification task minimizing cross-entropy. It turns out that the optimal discrimination function $D^*(z)$ is the posterior probability to be real.

We rely on the following mild assumptions, which are typical for the analysis for generative adversarial games\cite{arjovsky2017wasserstein,mescheder2018training}.
\begin{assumption}[The Generative Adversarial Game]
(i). \textit{Infinite Capacity:} 
The parameter space of $\theta$ has sufficient capacity to represent both the optimal discriminator function and the true data distribution;
(ii). \textit{Differentiability:} 
The densities $p_\theta(z)$ and $p_T(z)$ continuous, and $p_\theta(z)$ is differentiable with respect to $\theta$.
\end{assumption}
	
\begin{proposition}
\label{prop:discriminator}
For a fixed generator distribution $p_{G}(z)$, the optimal discriminator $D^*(z)$ trained via \eqref{eq:d1} is making discrimination by:
$D^*(z)=\nicefrac{p_T(z)}{p_T(z) + p_G(z)}$.
\end{proposition}

The proof is provided in Appendix~\ref{app:t1}.
This result shows that through discrimination update, the model implicitly learns the real and generated data distributions and discriminate by the posterior probability.


Next, we show that performing generation update against above optimal discriminator leads to the recovery of the true data distribution.

\begin{proposition}
	\label{prop:generator}
	Given the optimal discriminator $D^*$, the global minimum of the generator objective is achieved if and only if the generated distribution matches the real distribution, i.e., $p_{G^*}(z) = p_T(z)$.
\end{proposition}

The proof is provided in Appendix~\ref{app:t2}.
This result suggests the adversarial feedback from $D^*$ forces the generator $p_{G^*}$ to converge to the true data distribution $p_T$.

\subsection{From Few-Shot Generation to Zero-Shot Understanding}

A natural question when interpreting SGALM is whether its reliance on few-shot generation and self-discrimination merely leads to context-level imitation rather than genuine domain understanding. Since synthetic samples are generated through ICL, it is not immediately obvious why the resulting model should generalize beyond the specific prompts used during training. Here, we link the global distributional convergence to the desired ready-to-use capability: understanding the target domain $p_T(z)$ via the ability to respond to $\forall z_{<t}$ (e.g., zero-shot instruction following), rather than mimicking provided examples. To do so, we formalize the role of the prompting mechanism described in \eqref{eq:gen}.

The ICL capacity in LLM can be formalized as performing Bayesian inference (MAP) with prior of pre-training distribution.
We assume this holds based on existing literature\cite{akyurek2022learning,ahn2023transformers,bai2023transformers,wu2023many,li2023transformers,wu2025context}.
\begin{assumption}[ICL Capacity in LLM]\label{assumption}
Let $f_{z^i}(\hat{z})=p(z^i\mid \hat{z})$ be the likelihood function of domain $\hat{z}$, 
and $p(\hat{z}\mid\theta)$ be the prior distribution``memorized" by $\theta$.
Given few-shot examples $\{z^i\}_{i=1}^{n\text{-shot}}$ to construct a prompt as \eqref{eq:gen}, 
the LLM has output distribution
$p_{\theta}(\cdot\mid \{z^i\}_{i=1}^{n}) = p_{\hat{z}^*}(\cdot)$,
where
${\hat{z}^*}=\arg\max_{{\hat{z}}} \prod_{i=1}^nf_{z^i}(\hat{z})p(\hat{z}\mid\theta)$.
\end{assumption}

In plain words, given a few examples, the LLM identifies the underlying domain that best explains those examples given its pre-trained knowledge, and subsequently generates outputs according to that domain's distribution.
We assume the LLM we use are equipped with such capacity and would preserve during the fine-tuning, 
and for any domain $\hat{z}$, 
the distribution in the domain $ p(z\mid\hat{z})$ is continuous.
Therefore, we have the following theorem,
with proof in Appendix~\ref{app:t3}.
\begin{theorem}\label{thm:1}
	The distribution of the converged SGALM satisfies
	$p_{\theta^*}(\cdot)=p_T(\cdot)$.
\end{theorem}

Since language models factorize the joint distribution autoregressively as
$
p(z) = \prod_{t=1}^{|z|} p(z_t \mid z_{<t}),
$
matching the joint distribution implies that all conditional distributions are also matched,
i.e.,
$p_{\theta^\ast}(z_{\ge t} \mid z_{< t}) =
p_T(z_{\ge t} \mid z_{< t})$, $\forall t$.
%

\section{Comparison with Existing Methods}
\label{subsec:comparison}

Before presenting quantitative results, we compare SGALM with prior alignment/fine-tuning methods that mitigate the scarcity of human annotations (Table~\ref{tab:comparison_works}). Existing approaches differ in external dependency, generation scope, and supervision. Iterative synthetic-data pipelines repeatedly sample, filter, and fine-tune, typically using external rewards or ground-truth verification to construct preference signals (e.g., Iterative-DPO \cite{xiong2024iterative}, ReST \cite{gulcehre2023reinforced}, Iterative-RPO \cite{pang2024iterative}). In contrast, self-improvement methods remove external feedback but introduce new assumptions: Self-Rewarding \cite{yuan2024self} relies on reliable LLM-as-a-judge scoring and may suffer reward drift, while SPIN \cite{chen2024self} assumes real responses always dominate model generations ($y>y'$), ignoring quality variance. SGALM instead introduces a fully adversarial loop that requires only a dataset (with $z$ not restricted to $(x,y)$ pairs), avoiding external reward models while continuously grounding discrimination with real samples. It does not rely on any assumption on the existing LLM's generation/discrimination quality, i.e., it is applicable to train a model from scratch with enough training to create a synthetic data engine.

\section{Empirical Results}
\label{sec:experiments}

Here, 
we evaluate the performance of SGALM, comparing it with existing fine-tuning methods for alignment.
We also analyze the training dynamics and the model's capability to serve as a synthetic data engine.

\subsection{Experimental Setup}

\textbf{Models and Datasets.} We use {Qwen2.5-3B-Instruct} \cite{yang2024qwen2} as our base model. 
Following the setting mentioned in Section~\ref{sec:set}, the model is fine-tuned given a domain-specific training set, and evaluated on the testing set from the same domain.
We evaluate on three widely-used benchmarks respectively,  with various domains and training set sizes:
(i) \textit{GSM8K} \cite{cobbe2021gsm8k}: Math word problems, with a 7.47k training set;
(ii) \textit{ARC-Challenge} \cite{allenai:arc}: Multiple-choice science question answering problems, with a 1.12k training set;
(iii) \textit{MBPP} \cite{austin2021program}: Python programming problems, with a 0.12k training set.


\textbf{Baselines.} We compare SGALM with the following alignment fine-tuning baselines:
(i) \textit{Base}: The original model evaluated in the zero-shot setting.
(ii) \textit{SFT}: Supervised fine-tuning on the target training set.
(iii) \textit{Self-Instruct} \cite{wang2023self}: SFT on synthetic instruction--response pairs generated via ICL.
(iv) \textit{Self-Rewarding} \cite{yuan2024self}: Iterative self-play using LLM-based self-evaluation to construct preference pairs.
(v) \textit{SPIN} \cite{chen2024self}: Iterative self-play that contrasts real responses with model-generated responses.
(vi) \textit{Iterative-RS}: Iterative self-play with reject sampling over multiple generated responses, adapted from ReST \cite{gulcehre2023reinforced} and Iterative-DPO \cite{xiong2024iterative} without external rewards.
(vii) \textit{Iterative-RPO} \cite{pang2024iterative}: An extension of Iterative-RS with an additional NLL objective.

\textbf{Implementation.} All ICL-generation is provided with 4-shot examples. For all iterative methods, in each iteration, we generate synthetic sample (pairs) with the same number as the true training set (7.47/1.12/0.12k). More details are provided in Appendix~\ref{app:imp}. Code will be provided to public.

\subsection{Main Results}

Table~\ref{tab:main_results} 
presents the results.
In the table, ``N.A." means not applicable. 
As without external feedback, we could not define "consistency" of answers  on tasks like coding, which is a challenge faced in many open-ended tasks beyond our benchmarks.
SFT, Self-Instruct, and SPIN on MBPP, Self-Rewarding, Iterative-RS and Iterative-DPO on ARC has lower performance than Base.
This indicates that fine-tuning, even with moderate configurations, led to performance degradation below the zero-shot (Base) model, as the base model is already powerful and the training set is too small. 
While one could argue that a zero learning rate would yield performance equal to Base, such a result is trivial to be reported.

SGALM consistently outperforms existing methods across all benchmarks. 
As for efficiency, 
SGALM does not incur significantly higher training costs than other self-play baselines, as they all follow the iteratively generation-updating pipeline and converge within few iterations. The cost is reported in Appendix~\ref{app:cost}.

\begin{table}[t]
\caption{Performance comparison. N.A. means not applicable.}
\label{tab:main_results}
\small
\setlength\tabcolsep{1.5pt}
\vspace{-10px}
\begin{center}
			\begin{tabular}{lccc}
				\toprule
				Method & GSM8K & ARC & MBPP \\
				\midrule
				Base & $67.95\!\pm\! 1.24$ & $48.04\!\pm\! 1.46$ & $72.37 \!\pm\! 1.56$\\
				SFT & $68.61\!\pm\! 1.26$ & $51.39\!\pm\! 1.43$ & $71.60\!\pm\!1.95$   \\
				Self-Instruct & $68.40\!\pm\! 1.23$ & $51.65\!\pm\! 1.48$ & $71.98\!\pm\!1.74$   \\
				Self-Rewarding & $68.83\!\pm\! 1.24$ & $48.01\!\pm\!1.48$ & $73.15 \!\pm\! 1.95$ \\
				SPIN & $69.29\!\pm\! 1.24$ & $48.85\!\pm\! 1.42$ & $71.21\!\pm\!1.10$   \\
				Iterative-RS & $69.77\!\pm\! 1.26$ & $47.86\!\pm\!1.43$  & N.A. \\
				Iterative-RPO & $69.60\!\pm\! 1.26$ & $47.92\!\pm\!1.45$  & N.A.\\\midrule
				{SGALM (D-Only)} &${69.83\!\pm\! 1.23}$ & ${49.24\!\pm\! 1.46}$& ${73.15\!\pm\! 1.56}$\\
				{SGALM (G-Only)} &${70.58\!\pm\! 1.25}$ & ${52.58\!\pm\! 1.41}$& ${73.54\!\pm\! 1.95}$\\
				{SGALM (w/o S)} &${72.65\!\pm\! 1.24}$ & ${51.65\!\pm\! 1.45}$& $\mathbf{76.26\!\pm\! 1.10}$\\
				\textbf{SGALM } &$\mathbf{72.71\!\pm\! 1.23}$ & $\mathbf{53.02\!\pm\! 1.45}$& ${75.10\!\pm\! 1.95}$ \\
				\bottomrule
			\end{tabular}
\end{center}
\vspace{-10px}
\end{table}

\subsection{Ablations and Analysis}
To understand the contribution of each component objective in SGALM, we implement the following three variants: \textit{D/G-only }that only performs discrimination/generation update respectively, while removing another objective, and \textit{w/o S} that use two separate models as discriminator and generator.
To better the learning process, we visualize the performance over iterations, shown as Figure~\ref{fig:training_curve}.

\subsubsection{Diverse Synthetic Samples Mitigates Over-Fitting}

In Figure~\ref{fig:training_curve}, the first thing to notice is that data scarcity is an important issue. 
That even on GSM8K which has the largest training set (7.5k) in our experiment, 
all baselines starts to over-fit before 4 iterations with degrading performance.
We can also find that baselines with generation by given $x$, new $y$ (SPIN, Iterative-RS, Iterative-RPO) severely suffer that their performance drops under base before iteration 4. While baselines that can generate new $(x,y)$ (Self-Instruct, Self-Rewarding, SGALM) stay above base, as such generation scope generate much more diverse synthetic samples that helps to generalize and mitigate the data-scarcity.


\begin{figure*}[htbp]
	\centering
	\begin{subfigure}[]{0.32\textwidth}
		\centering
		\includegraphics[width=\textwidth]{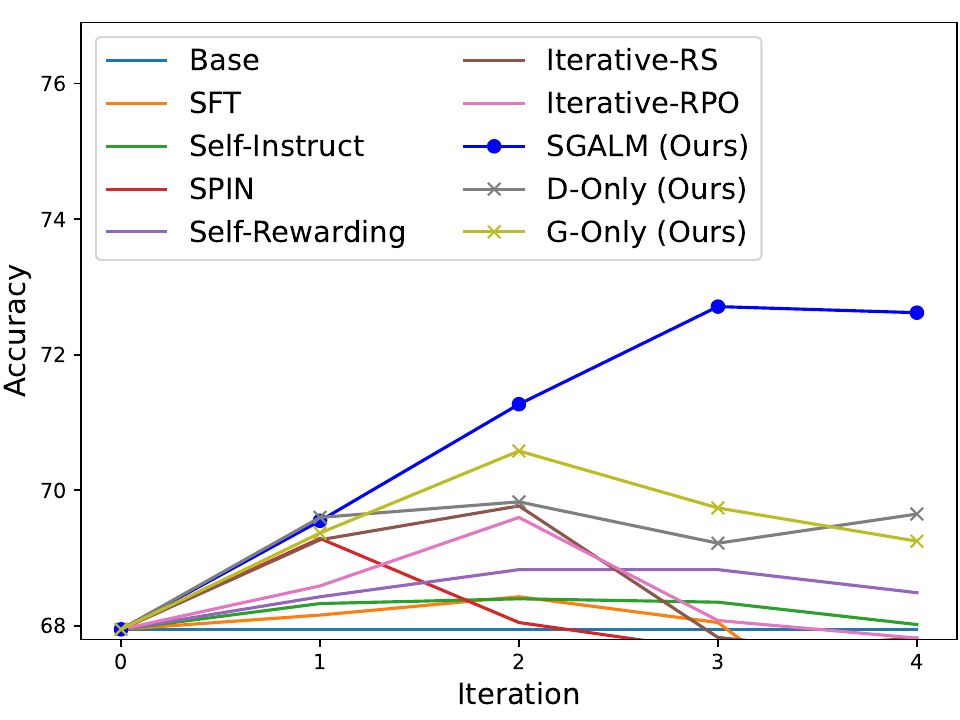}
		\caption{GSM8K (Math, 7.5k)}
		\label{fig:gsm8k}
	\end{subfigure}
	\hfill 
	\begin{subfigure}[]{0.32\textwidth}
		\centering
		\includegraphics[width=\textwidth]{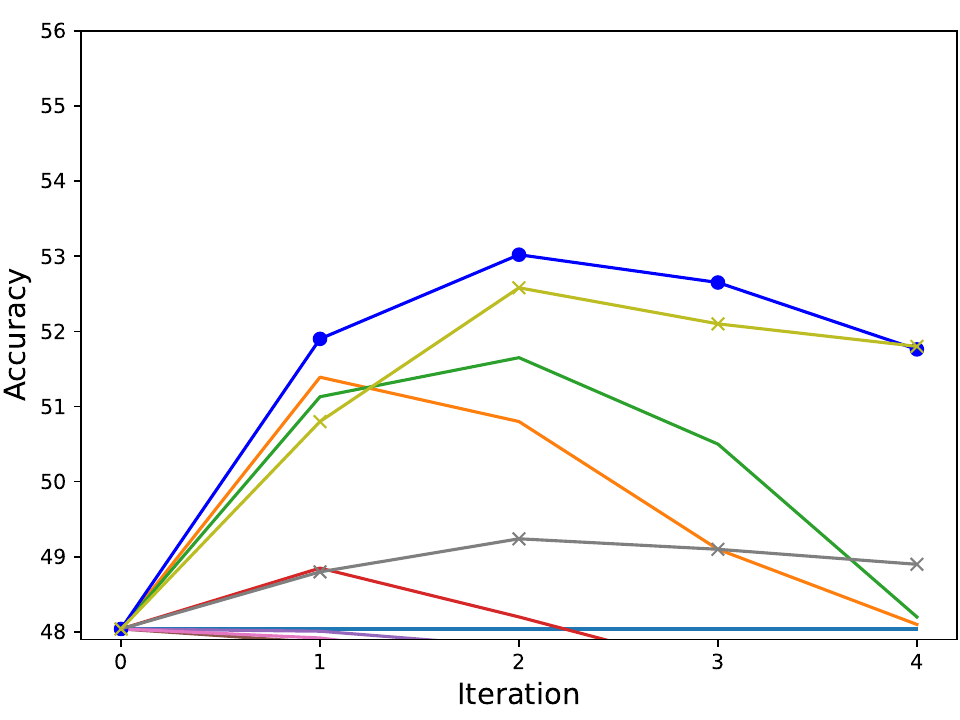}
		\caption{ARC (Science, 1.1k)}
		\label{fig:arc}
	\end{subfigure}
	\hfill 
	\begin{subfigure}[]{0.32\textwidth}
		\centering
		\includegraphics[width=\textwidth]{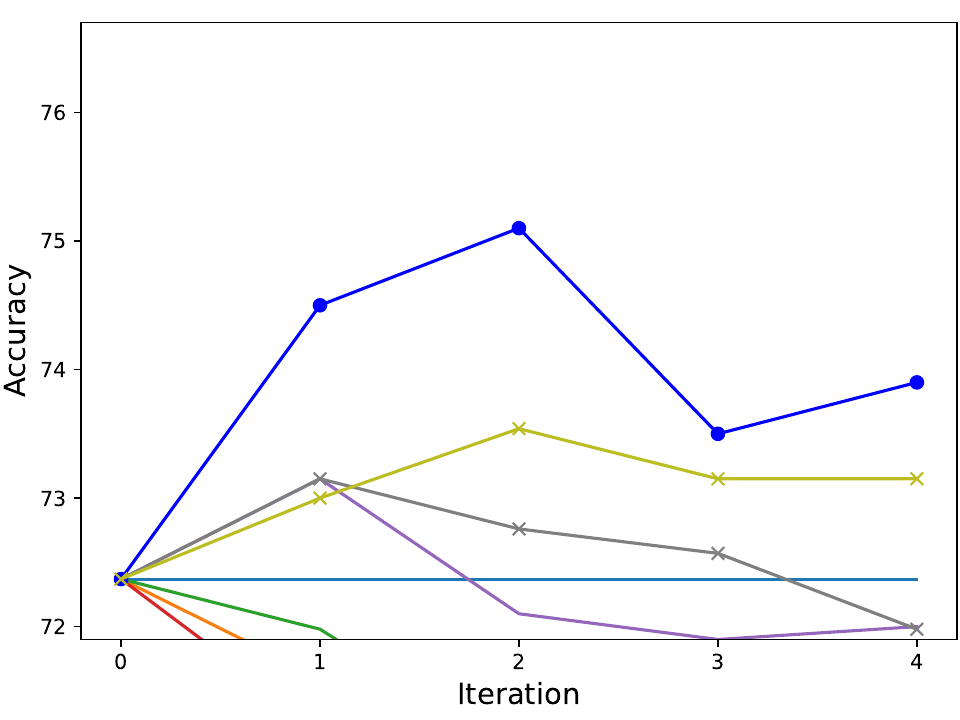}
		\caption{MBPP (Coding, 0.1k)}
		\label{fig:mbpp}
	\end{subfigure}
\vspace{-5px}
	\caption{Test accuracy across training iterations on GSM8K, ARC-Challenge, and MBPP for SGALM, its variants, and baseline methods.}
	\label{fig:training_curve}
	\vspace{-10px}
\end{figure*}

\subsubsection{The Duality of Generation and Discrimination}

We also find that D-only and G-only turn out to be strong baselines. D-only is only performing discrimination update 
(\eqref{eq:d1} with \eqref{eq:gd-d}) in every iteration. As it updates $\theta$, 
which is shared by the generation process, it also updates the generation distribution, as it has a KL divergence between generation distribution $\sum_{i=1}^4\text{KL}(p^G_{\theta^i}||p^G_{\theta^{i-1}})=5.89$. 
Though it is is not as significant as variants with generation update ($11.92$ for SGALM, $12.60$ for G-only), it is significantly larger than 0. 
And the stable and considerable performance improvement suggests that discrimination capacity overlaps with the desired understanding capacity.

G-only is only performing generation update 
(\eqref{eq:g1} with \eqref{eq:gd-g}) in every iteration. Such update framework is similar with Self-Rewarding, 
but with different way to discrimination: Self-Rewarding takes multiple generated $y$ for a certain $x$, 
and ask the LLM to give a goodness score and takes samples with the highest/lowest as win/lose pair; 
SGALM discriminates by a continuous score based on output distribution \eqref{eq:dis2}, 
for each $z=(x,y)$. Such discrimination way is not only optimizable, 
but also more flexible and grounded by the contrast between the semantic of ``Real" and ``Fake".
which result in the significant advantage of G-only over Self-Rewarding in Figure~\ref{fig:training_curve}. 

As for w/o S, which implements the discriminator and generator with two separate models, it (generator) shows comparable performance with SGALM. 
However, w/o S  requires about two times training cost, due to introducing an additional model. While the other baselines and variants require only one single model, 
we analysis w/o S with SGALM, not comparing with others. Comparing their performance, we infer that the relation between model capacity and task difficulty matters. 
SGALM co-benefits from the 
%
discrimination and generation in general intelligence, which could mutually promote each other, but limited to the parameter capacity to achieve both objectives with a single model. 
Thus, SGALM outperforms its variant without $\mathcal{S}$ on relatively simple tasks (ARC and GSM8K), while underperforming on the more challenging MBPP benchmark. 
Consequently, we infer that the larger the model is, the greater SGALM would benefit.

\subsubsection{The Unbalanced Generation and Discrimination Capacities}
A special phenomenon noticed in SGALM is the unbalance between discrimination and generation scores.
Figure~\ref{fig:p} shows the average $p^{\text{real}}_\theta$ on each iteration. Cases are provided in Appendix~\ref{app:case}.
Initially (iteration 0), the model assigns similarly near perfect average scores ($0.9923/0.9854$) for both real/generated samples. 
This is because most generated samples are generally reasonable and coherent, and the model has not been reinforced to discriminate.
After one iteration SGALM keeps giving average score $p^{\text{real}}_\theta(z)>0.9$ for real samples, 
while average score $p^{\text{real}}_\theta(z')<0.2$ for generated samples.
We infer this is because some patterns can be discriminated as AI-generated. 
Though we can not find seeming patterns in cases, it turns out to be easy to discriminate once the model is trained with discrimination objective (SGALM and D-only). 
While typical GAN can adjust the capacity of discriminator or generator repetitively to make them comparable for effective gradient update, which can not be done in SGALM, we managed to address in the following way.

\begin{figure}[t]
\centering
\includegraphics[width=0.48\textwidth]{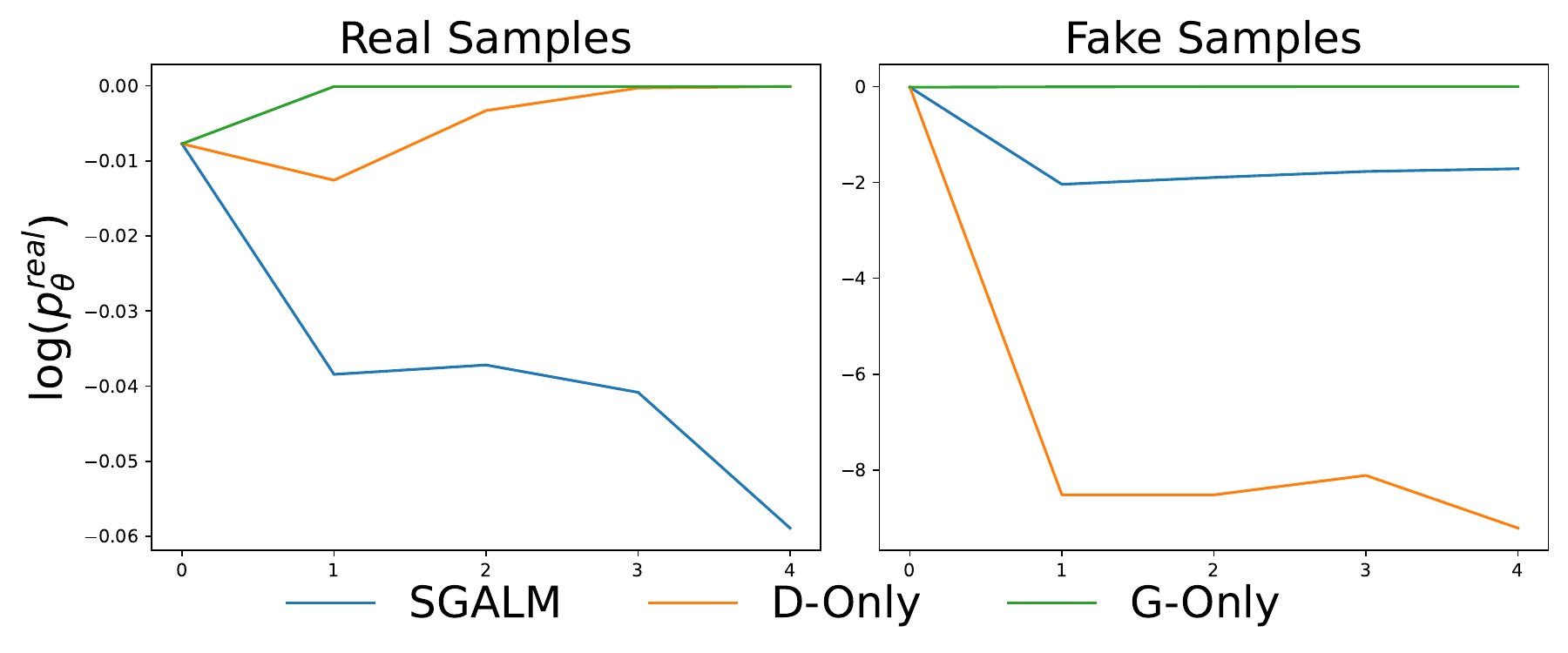}
\caption{$\log(p^{\text{real}}_\theta)$ vs Iteration of real and generated samples resp.}
\label{fig:p}
\vspace{-15px}
\end{figure}

Such unbalance is not fatal. It is not caused by memorization, i.e., over-fitting the training set that only gives high score to seen samples. We show the distribution of $p^{\text{real}}_\theta$ of well-trained SGALM in Figure~\ref{fig:count}. 
We can see that 
(i)~unseen real samples (test set) has very similar distribution with seen real samples (training set), which means they are not scored by memorization, but generalizable patterns; 
(ii)~the scores among generated samples are distinguishable, so we can implement the generation update with standardized rewards among generated batch, to make generation update effective and stable, which can be verified by the stable KL-divergence between generation distributions between two iterations (Appendix~\ref{app:kl}) and the growing of score of generated samples of SGALM (blue curve in Figure~\ref{fig:training_curve} right).


\subsection{Evidence against Mode Collapse and for ICL Capacity}\label{sec:collapse}
A historical challenge in training GAN, particularly with discrete text, is mode collapse \cite{kossale2022mode}, where the generator outputs identical or repetitive samples to exploit the discriminator. We observe no such phenomenon in SGALM. While generated samples frequently share common prefixes, this behavior reflects domain adherence rather than collapse. For example in GSM8K, $28\%$ of generated samples (iteration 4 model) start with ``\texttt{Q: A bakery}", but as the context grow longer, they gradually become distinguishable: $10\%$ of generated samples start with ``\texttt{Q: A bakery makes}" and $1.25\%$ of generated samples start with ``\texttt{Q: A bakery makes cupcakes}". This could be explained as the consequence of the common pattern of the few-shot examples from the same training set, and memorized knowledge in the model. Because the iteration 0 model also shows similar phenomenon.

Furthermore, this phenomenon provides empirical support for Assumption~\ref{assumption} (ICL Capacity), which posits that the model performs Bayesian inference to identify the target domain distribution from few-shot examples. The frequent occurrence of common prefixes reflects the model’s robust extraction of the shared prior knowledge and structural patterns inherent in the provided examples $Z_{ctx}$. Essentially, the model effectively "locks on" to the specific domain definition (the common prefix) dictated by the prompt ($p(\hat{z}\mid\theta)$). Crucially, the subsequent divergence of the sequences confirms that the model is not simply memorizing a single optimal path (mode collapse), but is instead sampling diverse trajectories related to the context-specific distribution ($\prod_{i=1}^nf_{z^i}(\hat{z})$)
once the context is established. This duality—rigid adherence to the domain pattern (prefix) coupled with flexible generation (suffix)—validates that ICL successfully serves as a mechanism for distribution recovery rather than mere imitation.

\begin{figure}[t]
\centering
\includegraphics[width=0.40\textwidth]{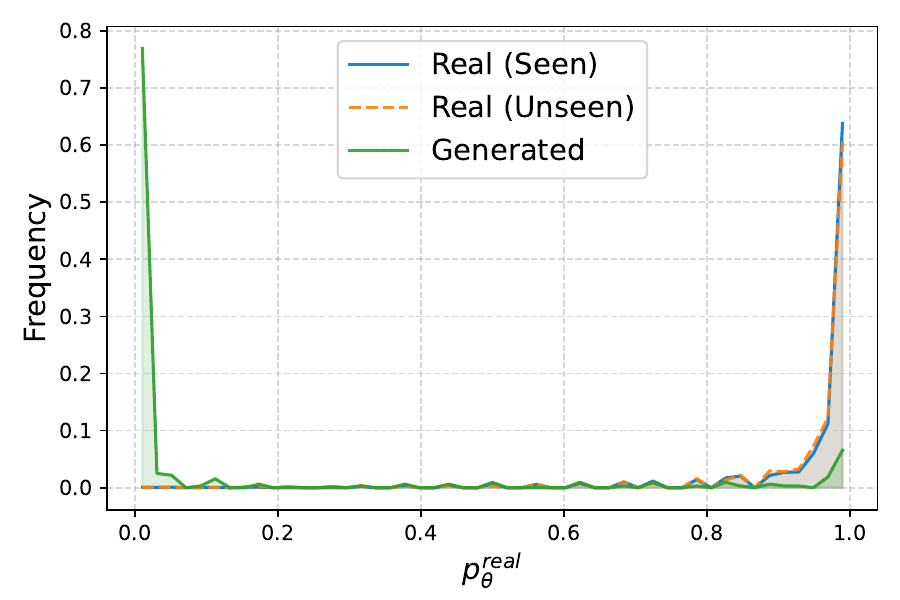}
\vspace{-5px}
\caption{Distribution of $p^{\text{real}}_\theta$ of seen real samples (training set), unseen real samples (test set), and generated samples.}
\vspace{-10px}
\label{fig:count}
\end{figure}

\subsection{SGALM as Synthetic Data Engine}
After aligned to the target distribution, SGALM can serves as a synthetic data engine:
we can draw few-shot examples from the training set $\mathcal{D}$ to form $Z_{\text{ctx}}$ in \eqref{eq:gen} to generate synthetic samples. 
Note that due to the very large combination and permutation numbers of real examples, 
and the inherent and controllable randomness in the generation process, we can draw almost infinite number of synthetic samples. 
And it can filter the synthetic samples by self-discrimination reality score $p^{\text{real}}_\theta(z')$.

\paragraph{Baselines}
We compare it with the the following baselines for synthetic data generation: 
(i) \textit{Self-Instruct} \cite{wang2023self}: Generate by ICL-generate using the base model.
(ii) \textit{RFT}: Generate multiple responses for each question using the base model, and filter by reject sampling by taking the majority answer.
(iii) \textit{Self-Instruct+RFT}: First generate new questions by Self-Instruct, then generate responses by RFT.
Another (smaller) model Qwen2.5-1.5B-Instruct will be aligned by SFT, on the synthetic enriched dataset containing the original training set of GSM8K, and certain number of generated samples.
All ICL-generation is with 4-shot examples randomly sampled from training set. For RFT, we generate 8 candidate to determine the majority answer, and choose random one leading to the majority answer.
For SGALM, to keep similar generation cost, we generate 8$\times$synthetic samples we want to keep, and then filter the top $\frac{1}{8}$ by highest self-discrimination reality score. We evaluate base model without SFT, SFT with real train set,  SFT with real train set and 7.47k/14.94k/29.88k (=Real+ 1$\times$/2$\times$/4$\times$ Syn.) synthetic samples generated by the 4 methods respectively.

The results are provided in Table~\ref{tab:synthetic}. 
The most critical finding is the trend observed when scaling the synthetic data volume from $1\times$ to $4\times$. 
SGALM is the only method that exhibits a positive scaling law, with performance continuously improving from $58.25\%$ to $59.73\%$. In stark contrast, standard Self-Instruct suffers from severe model collapse, where performance degrades drastically from $57.83\%$ to $54.25\%$—falling well below the baseline trained on real data. Similarly, RFT and Self-Instruct+RFT fail to sustain improvements, stagnating or declining at larger data scales. This demonstrates that SGALM's learned discriminator acts as a robust filter against ``toxic" or hallucinatory samples that typically accumulate in self-generated datasets, effectively converting the quantity of synthetic data into quality improvements.
Meanwhile, at all the synthetic scales ($1\times$/$2\times$/$4\times$ ), the model fine-tuned on SGALM-generated data achieves better performance surpassing both the baseline trained on real data ($57.16\%$) and other generation methods. This indicates that the adversarial discrimination process effectively selects high-quality samples that contribute more to learning than existing model generations.

\begin{table}[t]
	\caption{Comparison of SFT Qwen2.5-1.5B-Instruct on GSM8K, by synthetic enriched dataset generated with different methods.}
	\label{tab:synthetic}
	\centering
	\setlength\tabcolsep{2pt}
	\small
	\begin{tabular}{c|ccc}
		\toprule
		Base & \multicolumn{3}{c}{$50.72\pm1.38$} \\\midrule
		SFT with Real & \multicolumn{3}{c}{$57.16\pm1.36$} \\\midrule
		SFT with Real+  &  Syn.(1$\times$) &Syn.(2$\times$) & Syn.(4$\times$)  \\\midrule
		Self-Instruct  & $57.83\pm1.35$ &$57.04\pm1.36$ &$54.25\pm1.38$  \\
		RFT  & $58.06\pm1.36$&$58.03\pm1.34$ &$57.65\pm1.34$  \\
		Self-Instruct+RFT  & $57.95\pm1.36$ &$58.82\pm1.36$ &$57.80\pm1.37$ \\
		\textbf{SGALM  (Ours})  & $58.25\pm1.34$ &$59.56\pm1.36$&$\mathbf{59.73\pm1.34}$  \\
		\bottomrule
	\end{tabular}
\vspace{-10px}
\end{table}

\section{Conclusion}
In this work, we presented SGALM, a unified framework that redefines fine-tuning as a self-contained adversarial game within a single LLM. By utilizing the model's own ICL and output distributions, SGALM removes the need for external supervision or heuristic baselines, providing a fully grounded alignment process. Our theoretical analysis confirms that this approach converges to the true data distribution, effectively bridging the gap between few-shot generation and zero-shot understanding. SGALM demonstrates state-of-the-art results, serving effectively both as an alignment tool and a scalable synthetic data engine. Acknowledging the challenges of adversarial training, we view this work as a primary step toward robust generative alignment. Future research will focus on enhancing stability through noise injection, optimizing iteration schedules, and scaling the framework to a broader range of models and domains.


\section{Impact Statement}
This paper presents work whose goal is to advance the field of machine learning. There are many potential societal consequences of our work, none of which we feel must be specifically highlighted here.

\bibliography{example_paper}
\bibliographystyle{icml2026}

\newpage
\appendix
\onecolumn

\section{Related Works}\label{app:rw}
%
Here we discuss about that using LLMs to provide their own training signal has become a pivotal strategy for overcoming the bottleneck of high-quality human annotation. And discuss about other related works about adversarial training in LLMs and early text GANs.

\subsection{Synthetic Data Generation}
The paradigm began with Self-Instruct \cite{wang2023self,honovich2023unnatural}, which demonstrated that a weak model could be fine-tuned on instruction-output pairs generated by itself. Subsequent works like Alpaca \cite{dubois2023alpacafarm} and Vicuna \cite{chiang2023vicuna} popularized this distillation approach. More recent methods focus on increasing data complexity and diversity; for example, WizardLM \cite{xu2024wizardlm} employs "Evol-Instruct" to progressively rewrite instructions into more complex forms. However, these methods primarily rely on static distillation from a fixed teacher, lacking a dynamic feedback loop to correct errors during training, and the error would accumulate to collapse \cite{shumailov2024ai}.

\subsection{Self-Play and Iterative Refinement}
To enable models to improve themselves without a stronger teacher, Self-Play mechanisms have been introduced. Reinforcement Learning from AI Feedback (RLAIF) \cite{lee2023rlaif,bai2022training} replaces human preference labeling with model-generated preferences, scaling up the RLHF pipeline. Self-Rewarding Language Models \cite{yuan2024self} integrate the reward modeling capability directly into the policy model, allowing the LLM to act as its own judge during training.
Similarly, SPIN \cite{chen2024self} proposes a method where the model plays against its previous iteration. By treating the previous epoch's model as a generator of "negative" samples and the ground truth as "positive," SPIN optimizes a DPO-like objective to widen the gap between the two. Though also inspired by GAN, SPIN makes discrimination by assuming a strict dichotomy where model generations are always inferior to data, potentially penalizing valid but diverse responses. Our work (SGALM) addresses this by introducing a learnable discriminator that assesses quality dynamically rather than relying on a static assumption.

Related to self-play, iterative training methods like ReST \cite{gulcehre2023reinforced}, Iterative DPO \cite{xiong2024iterative} and Iterative RPO \cite{pang2024iterative} involve repeated cycles of sampling, ranking/filtering, and fine-tuning. While these methods improve performance, they typically rely on fixed reward models or ground-truth verification (e.g., in math problems), which limits their applicability in open-ended domains compared to our adversarial approach.

\subsection{Adversarial Training in LLMs}
With the rise of LLMs, adversarial concepts have resurfaced with a focus on robustness and alignment rather than generation from scratch. Red Teaming \cite{perez2022red} employs an adversarial LLM to generate test cases that provoke harmful outputs from a target LLM, serving as a distinct "attacker." In the context of alignment, f-GAN formulations have been theoretically connected to divergence minimization in language models \cite{go2023aligning}.
Most closely related to our work is the concept of using a discriminator to guide decoding. Methods like Discriminator-Guided Generation train a classifier to rerank or guide the beam search of an LLM. However, these discriminators are usually small, separate models (e.g., BERT-based) or fixed reward models. SGALM differs by unifying the generator and discriminator into a single LLM that iteratively updates both roles, leveraging the high-level reasoning capabilities of modern LLMs to perform nuanced discrimination beyond simple binary classification.

\subsection{Text GAN}
GANs have achieved immense success in continuous domains like computer vision but have historically struggled with discrete text generation.
Early Text GANs. The discrete nature of text disrupts the gradient flow from the discriminator to the generator. Early solutions like SeqGAN \cite{yu2017seqgan}, MaliGAN \cite{che2017maximum}, and LeakGAN \cite{guo2018long} employed Reinforcement Learning (specifically Policy Gradient) to bypass the non-differentiable token generation. However, these models famously suffered from mode collapse and training instability \cite{zhang2017adversarial}, and generally underperformed compared to maximum likelihood estimation (MLE) in the pre-LLM era.

\section{Theoretical Analysis}

We provide a theoretical justification for the convergence properties of the SGALM framework. We analyze the minimax game played by the unified model $\theta$, where the discriminator attempts to distinguish the real data distribution $p_T$ from the generated distribution $p_\theta$, and the generator attempts to minimize this distinction.

We rely on the following standard assumptions for the analysis.
{Infinite Capacity:} The parameter space of $\theta$ has sufficient capacity to represent both the optimal discriminator function and the true data distribution;
{Differentiability:} The densities $p_\theta(z)$ and $p_T(z)$ continuous, and $p_\theta(z)$ is differentiable with respect to $\theta$;
{Optimization:} The discrimination and generation updates converge in their respective alternating steps.

\subsection{Optimal Discrimination by Posterior Probability}\label{app:t1}

We first analyze the behavior of the discrimination update, i.e., updating \eqref{eq:d1} with \eqref{eq:gd-d}. This process is mathematically equivalent to a binary classification task minimizing cross-entropy. It comes out the optimal discrimination function $D^*(z)$ is the posterior probability to be real.

\begin{proposition}
For a fixed generator distribution $p_{G}(z)$, the optimal discriminator $D^*(z)$ trained via \eqref{eq:d1} is:
\begin{equation}
	D^*(z) = \frac{p_T(z)}{p_T(z) + p_G(z)}
\end{equation}
\end{proposition}

\begin{proof}
The discriminator $D(z)$ is derived from the probability of the token "Real". 
Training the model to distinguish between real samples $z \sim p_T$ and fake samples $z' \sim p_\theta$ corresponds to maximizing the standard binary log-likelihood objective $J(D)$:
\begin{equation}
	J(D) = \mathbb{E}_{z \sim p_T} [\log D(z)] + \mathbb{E}_{z' \sim p_G} [\log (1 - D(z'))]\nonumber
\end{equation}
We can rewrite this expectation in integral form over the data space $\mathcal{Z}$:
\begin{equation}
	J(D) = \int_{\mathcal{Z}} \left( p_T(z) \log D(z) + p_G(z) \log (1 - D(z)) \right) dz\nonumber
\end{equation}
To find the optimal $D(z)$ for any point $z$, we differentiate the integrand with respect to $D(z)$ and set the derivative to zero:
\begin{equation}
	\frac{p_T(z)}{D(z)} - \frac{p_G(z)}{1 - D(z)} = 0\nonumber
\end{equation}
Rearranging the terms yields:
\begin{align}
	D^*(z) &= \frac{p_T(z)}{p_T(z) + p_G(z)}\nonumber
\end{align}
\end{proof}
This result shows that through discrimination update, the model implicitly learns the real and generated data distributions and discriminate by the posterior probability.

\subsection{Optimal Generation from True Distribution}\label{app:t2}

Next, we show that performing generation update against above optimal discriminator leads to the recovery of the true data distribution.

\begin{proposition}
Given the optimal discriminator $D^*$, the global minimum of the generator objective is achieved if and only if the generated distribution matches the real distribution, i.e., $p_{G^*}(z) = p_T(z)$.
\end{proposition}

\begin{proof}
The generator aims to minimize the discriminator's ability to distinguish fake samples, which is equivalent to minimizing the value function $J(G | D^*)$. Substituting the optimal discriminator $D^*(z)$ into the objective:
\begin{align}
	J(G|D^*) &= \mathbb{E}_{z \sim p_T} \left[\log \frac{p_T(z)}{p_T(z) + p_G(z)}\right] \nonumber \\
	&+ \mathbb{E}_{z \sim p_G} \left[\log \frac{p_G(z)}{p_T(z) + p_G(z)}\right]\nonumber
\end{align}
This expression relates to the Jensen-Shannon Divergence ($JSD$). By factoring out $-\log 4$, we obtain:
\begin{equation}
	J(G|D^*)= -\log 4 + 2 \cdot \text{JSD}(p_T \| p_G)\nonumber
\end{equation}
Since the Jensen-Shannon Divergence is non-negative ($\text{JSD}(P\|Q) \geq 0$) and equals zero if and only if $P = Q$, the global minimum of the generator objective is achieved exactly when:
\begin{equation}
	p_{G^*}(z) = p_T(z)\nonumber
\end{equation}
\end{proof}
Thus, the adversarial feedback from $D^*$ forces the generator $p_{G^*}$ to converge to the true data distribution $p_T$.

\subsection{From Few-Shot Generation to Zero-Shot Understanding}\label{app:t3}

\begin{theorem}
The converged SGALM
$p_{\theta^*}(\cdot)=p_T(\cdot)$.
\end{theorem}


\begin{proof}
From Proposition~\ref{prop:generator} and Assumption~\ref{assumption}, we have
\begin{align}
	\forall \{z^i\mid z^i\sim p_T(z)\},~p_{\hat{z}^*}(\cdot)=p_T(\cdot).\label{eq:o}
\end{align}
We first prove \eqref{eq:o} holds i.i.f. $\hat{z}^*=T$, i.e., $\hat{z^*}$ is unique.
From definition, we have $\forall \hat{z}\neq T,$
$$ \exists \gamma\in(0,1), \frac{\prod_{i=1}^np(z^i\mid \hat{z})}{\prod_{i=1}^np(z^i\mid T)}\leq \gamma^n.$$
So $\forall n>\log_\gamma \frac{p(T\mid\theta^*) }{p( \hat{z}\mid\theta^*) }$,
\begin{align}\nonumber
	\frac{\prod_{i=1}^np(z^i\mid \hat{z})p(\hat{z}\mid\theta^*)}{\prod_{i=1}^np(z^i\mid T)p(T\mid\theta^*)}
	\leq\gamma^n \frac{p(\hat{z}\mid\theta^*)}{p(T\mid\theta^*)}
	<\frac{p(T\mid\theta^*) }{p( \hat{z}\mid\theta^*) }\frac{p(\hat{z}\mid\theta^*)}{p(T\mid\theta^*)}=1,
\end{align}
which means ${\hat{z}}\neq\hat{z}^*=\arg\max_{{\hat{z}}} \prod_{i=1}^nf_{z^i}(\hat{z})p(\hat{z}\mid\theta^*)$. So $\hat{z}^*=T$ is unique.

Now with $\forall  \{z^i\mid z^i\sim p_T(z)\}, p_{\hat{z}^*}(\cdot)=p_T(\cdot)$ and  $\hat{z}^*=T$, we prove $p(\hat{z}\mid\theta^*)=\delta_T(\hat{z})$, where $\delta_T$ is the Dirac delta function centered at $T$.
If $\exists \hat{z}\neq T, p(\hat{z}\mid \theta^*)>0$, then as $p(z\mid\hat{z})$ is continuous, $\exists z^j \in \text{supp} (p_T(z))$, 
\begin{align}\nonumber
	\frac{p(z^j\mid\hat{z})}{p(z^j\mid T)}&>\frac{p(T\mid\theta^*)}{p(\hat{z}\mid\theta^*)},\\\nonumber
	p(z^j\mid\hat{z})p(\hat{z}\mid\theta^*)&>p(z^j\mid T)p(T\mid\theta^*).
\end{align}
This means given $\{z^j\}$, which is possible to be drawn from $p_T(z)$, $\hat{z^*}=\arg\max_{{\hat{z}}} f_{z^j}(\hat{z})p(\hat{z}\mid\theta^*)\neq T$, which is contradictory to previous result. So $\forall \hat{z}\neq T, p(\hat{z}\mid \theta^*)=0$, which means $p(\hat{z}\mid\theta^*)=\delta_T(\hat{z})$.

Finally, with no examples (zero-shot $\emptyset$),
\begin{align}\nonumber
	p_{\theta^*}(\cdot)=p_{\theta^*}(\cdot\mid\emptyset)=p_{\arg\max_{{\hat{z}}}p(\hat{z}\mid\theta^*)}(\cdot)\\\nonumber
	=p_{\arg\max_{{\hat{z}}}\delta_T(\hat{z})}(\cdot)=p_T(\cdot).
\end{align}
\end{proof}

\section{Experiments}

\subsection{Implementation Details}\label{app:imp}
Here we provide implementation details and important hyperparameter settings in our experiments.
 For all iterative methods in our experiment, including SGALM, Self-Rewarding, SPIN, Iterative-RS, Iterative-RPO, in each iteration, we generate synthetic samples (pairs) with the same number as the true training set (7.47/1.12/0.12k). And train the model for one epoch per iteration, with linear decaying learning $lr_{start}=0.5\times lr_{end}$ with empirically best $lr_{start}$ chosen from $\{1\times10^{-6},5\times10^{-7},3\times10^{-7}\}$. Batch size is set to $64$. Except SFT and Self-Instruct are optimized with AdamW (standard supervised fine-tuning, default), the other methods are optimized with RMSProp (rewarding the policy, following \cite{rafailov2023direct}).
 
For the generation in SGALM, and all ICL-generation involved in other baselines, we randomly select 4-shot examples from training set as provided $Z_{\text{ctx}}$.
All ICL-generation is provided with 4-shot examples.

\subsection{Training Cost}\label{app:cost}
On 4 $\times$ NVIDIA A100 80G GPU, each iteration for GSM8K takes 0.4 hours generating and 1.1 hours training; each iteration for ARC takes 0.1 hours generating, 0.3 hours training; each iteration for MBPP takes 0.1 hours generating, 0.3 hours training.

\subsection{More Results}
\begin{figure}[ht]
	\centering
	\includegraphics[width=0.45\textwidth]{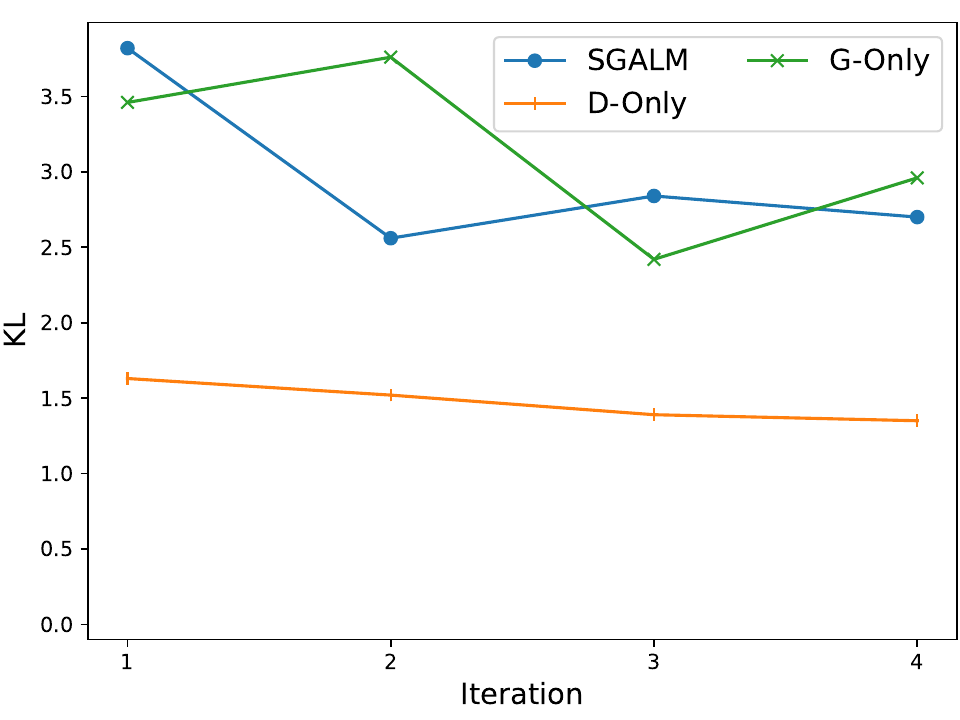}
	\caption{$\text{KL}(p^G_{\theta^i}||p^G_{\theta^{i-1}})$ for each iteration $i$.}
	\label{fig:kl}
\end{figure}
\subsection{Generation Distribution}\label{app:kl}
First, we show the KL-divergence between generation distribution of two iterations $i$ and $i-1$, i.e., $\text{KL}(p^G_{\theta^i}||p^G_{\theta^{i-1}})$, shown in Figure~\ref{fig:kl}.
The KL-divergence is estimated by $\frac{1}{N}\sum_{z\sim p^G_{\theta^{i-1}}}^{N}\frac{p^G_{\theta^{i}}(z)}{p^G_{\theta^{i-1}}(z)}\log\frac{p^G_{\theta^{i}}(z)}{p^G_{\theta^{i-1}}(z)}-\frac{p^G_{\theta^{i}}(z)}{p^G_{\theta^{i-1}}(z)}+1$.
We can find that D-only, although without explicit generation objective, also updates the generation distribution obviously.
There are two reasons. First, generation and discrimination are achieved by the same group of parameter $\theta$, updating $\theta$ would certainly affect the generation distribution. Second, discrimination better would  lead to better understanding and generation in human-like general intelligence, which is exactly the property SGALM aims to utilize, supported by the fact that $p_\theta^\text{real}(z')$ does increase after iteration 1 shown in Figure~\ref{fig:p}.

\subsection{Generation and Discrimination Cases}
Here we provide some cases about the generation and discrimination.

\subsubsection{Prompt Template}
First,we provide the exact prompt we use in SGALM (specifically for Qwen2.5-3B-Instruct).
The generation prompt is provided as follows, where {\color{blue}blue} means official chat template provided by Qwen2.5-3B-Instruct,
 {\color{brown}brown} means the prompt we set for SGALM shared across all data, and the others (black) are specific data samples.

\begin{Verbatim}[breaklines=true, breakanywhere=true,breaksymbol=,commandchars=\\\{\}]
\color{blue}{<|im_start|>system}
\color{brown}{You are a helpful assistant. Follow the user's examples to complete the task.\color{blue}{<|im_end|>}
<|im_start|>user}
\color{brown}{Here are 4 examples. Please follow the pattern and provide a new example.


### Example 1:}
Q: The bakery has 8 indoor tables and 12 outdoor tables. Each indoor table has 3 chairs and each outdoor table has 3 chairs. How many chairs are there in total?
A: Let's think step by step. There are 8 x 3 = <<8*3=24>>24 indoor chairs.
There are 12 x 3 = <<12*3=36>>36 outdoor chairs.
In total there are 24 + 36 = <<24+36=60>>60 chairs
#### 60

\color{brown}{### Example 2:}
Q: Ken had fifty pencils, and he wanted to share some of them with his two friends, Manny and Nilo. Ken gave ten pencils to Manny and ten more pencils to Nilo than he gave to Manny. He kept the rest of the pencils. How many pencils did Ken keep?
A: Let's think step by step. Nilo received 10 + 10 = <<10+10=20>>20.
Thus, Ken gave a total of 10 + 20 = <<10+20=30>>30.
Therefore, Ken kept 50 - 30 = <<50-30=20>>20 pencils.
#### 20

\color{brown}{### Example 3:}
Q: A small store made 50 posters to advertise their products. Two-fifths of them are small posters while half of them are medium posters. The rest are large posters. How many large posters are there?
A: Let's think step by step. 50 x 2/5 = <<50*2/5=20>>20 posters are small.
And 50/2 = <<50/2=25>>25 posters are medium.
So 20 + 25 = <<20+25=45>>45 posters are not large.
Therefore, 50 - 45 = <<50-45=5>>5 posters are large.
#### 5

\color{brown}{### Example 4:}
Q: Earl started delivering newspapers on the first floor of a condominium building. He then went up 5 floors then went down 2 floors. He again went up 7 floors and found that he is 9 floors away from the top of the building. How many floors does the building have?
A: Let's think step by step. Earl was on the 1 + 5 = <<1+5=6>>6th floor after going up 5 floors.
When he went down 2 floors, he was on the 6 - 2 = <<6-2=4>>4th floor.
Since he went up 7 floors, he was then on the 4 + 7 = <<4+7=11>>11th floor.
Since he is 9 floors away from the top of the building, therefore the building has 11 + 9 = <<11+9=20>>20 floors.
#### 20

\color{brown}{### New Example:}
\color{blue}{<|im_end|>
<|im_start|>assistant}
\end{Verbatim}

And the discrimination prompt is as follows.
\begin{Verbatim}[breaklines=true, breakanywhere=true,breaksymbol=,commandchars=\\\{\}]
\color{blue}{<|im_start|>system}
\color{brown}{You are a content detector. Follow the user's instruction to analyze the given text to determine if it is Real (human-written) or Fake (LLM-generated). Answer with exactly one word: 'Real' (if human-written) or 'Fake' (if LLM-generated).\color{blue}{<|im_end|>}
<|im_start|>user}
Q: Francie saves up her allowance for several weeks. She receives an allowance of $5 a week for 8 weeks. Then her dad raises her allowance, and she receives $6 a week for 6 weeks. Francie uses half of the money to buy new clothes. With the remaining money, she buys a video game that costs $35. How much money does Francie have remaining after buying the video game?
A: Let's think step by step. When her allowance is $5 a week, Francie gets a total of $5 * 8 = $<<5*8=40>>40
When her allowance is $6 a week, Francie gets a total of $6 * 6 = $<<6*6=36>>36
The total amount of money Francie gets is $40 + $36 = $<<40+36=76>>76
After purchasing new clothes, she has $76 / 2 = $<<76/2=38>>38 remaining
After buying the video game, Francie has $38 - $35 = $<<38-35=3>>3 remaining
#### 3

\color{brown}{Is above text Real (human-written) or Fake (LLM-generated)? Answer with exactly one word: 'Real' (if human-written) or 'Fake' (if LLM-generated).}
\color{blue}{<|im_end|>
<|im_start|>assistant}
\end{Verbatim}

\subsubsection{Generation and Discrimination Cases}\label{app:case}
We show case study on GSM8K. We use iteration 0 model (original Qwen2.5-3B-Instruct) and iteration 4 model to generate and discriminate 320 cases respectively. We also provide discrimination cases on the real samples (seen in train set).

For iteration 0 model, the generated/real cases  has scores ($p_\theta^\text{real}$) with mean $0.9570/0.9816$, median $1.0/1.0$, std $0.1764/0.1060$, minimum $0.0052/0.0000$, maximum $1.0/1.0$.
One generated case with highest score ($1.0$) is
\begin{Verbatim}[breaklines=true, breakanywhere=true,breaksymbol=]
Q: Sarah earns $15 for each hour she works at her job. She also gets a $20 weekly bonus. If Sarah works 40 hours in a week, how much money will she earn including her bonus?\n\nLet's think step by step. Sarah earns $15 per hour and works 40 hours, so her earnings from hourly work would be 40 * 15 = $<<40*15=600>>600.\nAdditionally, she receives a $20 weekly bonus.\nTherefore, the total amount Sarah earns is 600 + 20 = $<<600+20=620>>620.\n#### 620
\end{Verbatim}
The generated case with lowest score ($0.0052$) is
\begin{Verbatim}[breaklines=true, breakanywhere=true,breaksymbol=]
Q: A jar contains 50 red marbles and 30 blue marbles. If a marble is randomly selected from the jar, the probability that it is blue is twice the probability that it is red. How many additional blue marbles need to be added to the jar so that the probability of selecting a blue marble becomes equal to the probability of selecting a red marble?\n\nA: Let's think step by step. Initially, there are 50 red marbles and 30 blue marbles, making a total of 50 + 30 = 80 marbles.\n\nLet \\( x \\) be the number of additional blue marbles to be added.\n\nAfter adding \\( x \\) blue marbles, the total number of marbles becomes \\( 80 + x \\).\n\nThe number of blue marbles will be \\( 30 + x \\).\n\nGiven that the probability of selecting a blue marble is twice the probability of selecting a red marble, we can set up the following equation:\n\n\\[\n\\frac{30 + x}{80 + x} = 2 \\times \\frac{50}{80 + x}\n\\]\n\nSimplify the right-hand side:\n\n\\[\n\\frac{30 + x}{80 + x} = \\frac{100}{80 + x}\n\\]\n\nCross-multiply to solve for \\( x \\):\n\n\\[\n(30 + x)(80 + x) = 100(80 + x)\n\\]\n\nExpand both sides:\n\n\\[\n2400 + 30x + 80x + x^2 = 8000 + 100x\n\\]\n\nCombine like terms:\n\n\\[\nx^2 + 110x + 2400 = 8000 + 100x\n\\]\n\nSubtract \\( 8000 + 100x \\) from both sides:\n\n\\[\nx^2 + 110x + 2400 - 8000 - 100x = 0\n\\]\n\nSimplify:\n\n\\[\nx^2 + 10x - 5600 = 0\n\\]\n\nSolve this quadratic equation using the quadratic formula \\( x = \\frac{-b \\pm \\sqrt{b^2 - 4ac}}{2a} \\), where \\( a = 
\end{Verbatim}
A real case with very low score $1.583\times10^{-4}$ is
\begin{Verbatim}[breaklines=true, breakanywhere=true,breaksymbol=]
Q: By the time Anne is two times as old as Emile, Emile will be six times as old as Maude. If Maude will be 8 years old, how old will Anne be?\nA: Let's think step by step. If Maude's age is 8 by the time Anne's age is four times Emile's age, Emile will be six times as old as Maude, which totals 6*8 = 48 years.\nIf Emile's age is 48 years old by the time Anne's age is twice her number, Anne will be 2*48 = <<48*2=96>>96 years.\n#### 96
\end{Verbatim}
Another real case with very low score $1.230\times10^{-4}$ is
\begin{Verbatim}[breaklines=true, breakanywhere=true,breaksymbol=]
Q: Jenny leaves her house at 8:00 with some cupcakes for her daughter's birthday party at school. She jogs to the school at 15 miles per hour. Half an hour later, her wife Anna realizes that Jenny used peanut butter frosting because she forgot one of the kids in the class had a severe peanut allergy. Jenny doesn't have her cell phone, so Anna leaves the house driving at 45 miles her hour to catch Jenny and warn her. How long does Anna spend traveling in minutes?\nA: Let's think step by step. We know that the time that Anna spends traveling is 20 minutes less than Jenny's travel time: a = j - 20\nWe also know that the distance both people travel is equal, so their speeds times their travel times must be equal: 15j = 45a\nSubstituting the first equation into the second equation, we get 15j = 45(j - 20)\nMultiplying through the parentheses we get 15 j = 45j - 900\nSubtracting 45 j from both sides, we get -30j = -900\nDividing both sides by -30, we get j = 30\nSince Anna travels 20 minutes less than Jenny, she traveled 30 minutes - 20 minutes = <<20-10=10>>10 minutes\n#### 10
\end{Verbatim}

For iteration 4 model, the generated/real cases  has scores ($p_\theta^\text{real}$) with mean $0.0699/0.9434$, median $0.0179/0.9922$, std $0.1245/0.1282$, minimum $0.0000/0.0015$, maximum $0.9062/1.0$.
The generated case with highest score ($0.9062$) is
\begin{Verbatim}[breaklines=true, breakanywhere=true,breaksymbol=]
Q: A school garden club plants flowers in rows. They planted 5 rows of tulips with 6 tulips in each row last week. This week, they decided to increase the number of rows by 2 but kept the same number of tulips per row. How many tulips will they plant this week?\n\nLet's think step by step. Last week, the club planted 5 rows of tulips with 6 tulips in each row, so they planted 5 * 6 = <<5*6=30>>30 tulips.\nThis week, they increased the number of rows by 2, so they will have 5 + 2 = <<5+2=7>>7 rows.\nThey still plan to plant the same number of tulips per row, which is 6 tulips.\nTherefore, this week they will plant 7 * 6 = <<7*6=42>>42 tulips.\n#### 42
\end{Verbatim}
The generated case with lowest score ($4.7981\times10^{-6}$) is
\begin{Verbatim}[breaklines=true, breakanywhere=true,breaksymbol=]
Q: A factory produces 3 types of widgets: Type A, Type B, and Type C. The production ratio of Type A to Type B to Type C is 5:3:2. If the factory produced a total of 360 widgets today, how many more Type C widgets must be added to reach a production ratio of 5:3:4?\n\nLet's think step by step. First, we determine the total parts of the original ratio: 5 + 3 + 2 = 10 parts.\nGiven that the total production is 360 widgets, each part corresponds to 360 / 10 = <<360/10=36>>36 widgets.\nFor the original ratio of 5:3:2, the number of Type A, Type B, and Type C widgets produced are:\nType A: 5 * 36 = <<5*36=180>>180\nType B: 3 * 36 = <<3*36=108>>108\nType C: 2 * 36 = <<2*36=72>>72\n\nTo achieve a new ratio of 5:3:4, let's denote the new number of Type C widgets as \\(x\\). This means the number of Type A and Type B widgets will remain the same (180 and 108 respectively), and the new ratio equation will be:\n\\[ \\frac{180}{108} = \\frac{180}{x} = \\frac{108}{x} \\]\n\nWe know that the new ratio should be 5:3:4, so:\n\\[ \\frac{180}{72+x} = \\frac{180}{x} = \\frac{108}{x} \\]\n\nSolving for \\(x\\):\n\\[ \\frac{180}{72+x} = \\frac{180}{x} \\]\n\\[ 180x = 180(72 + x) \\]\n\\[ 180x = 180 \\times 72 + 180x \\]\n\\[ 180x = 12960 + 180x \\]\n\\[ 0 = 12960 \\]\n\nIt seems there was a m
\end{Verbatim}
A real case with very low score $0.0015$ is
\begin{Verbatim}[breaklines=true, breakanywhere=true,breaksymbol=]
Q: There are twice as many cows in Devonshire as there are hearts on a standard deck of 52 playing cards. If there are 4 hearts on a card, calculate the total cost of the cows when they are sold at $200 each.\nA: Let's think step by step. If there are 4 hearts on a card, a standard deck of 52 playing cards will have 208 hearts.\nThere are twice as many cows in Devonshire as there are hearts on a standard deck of 52 playing cards, meaning there are 2*208 = <<2*208=416>>416 cows in Devonshire.\nThe total cost of the cows when they are sold at $200 each is 416*$200 =$<<416*200=83200>>83200\n#### 83200
\end{Verbatim}
Another real case with very low score $0.0019$ is
\begin{Verbatim}[breaklines=true, breakanywhere=true,breaksymbol=]
Q: Porter is a painter who creates beautiful paintings of wild animals.  Most recently, he painted a mural of a pack of wolves standing on a snowy-white mountainside underneath a full moon.  He put the painting up for auction and it sold for $1000 less than five times more than he had made on his previous painting.  If he received $44,000 for the sale of his most recent painting, how much, in dollars, did he make for selling his previous painting?\nA: Let's think step by step. If $44,000 is $1000 less than five times more than he had made on his previous painting, then $44,000+$1000 = $45,000 is five times what he made on his previous painting.\nIf $45,000 is five times what he made for his previous painting, then he made $45,000/5 = $<<45000/5=9000>>9,000 for selling his previous painting.\n#### 9,000
\end{Verbatim}

\subsection{No Mode Collapse Phenomenon Observed}\label{app:collapse}
Serious mode collapse, that two generated samples are identical, has not been observed in SGALM. One phenomenon to note is that the generated sample is likely to share a few common prefixes. For example, $28\%$ of generated samples (iteration 4 model) start with ``\texttt{Q: A bakery}", but as the context grow longer, they gradually become distinguishable: $10\%$ of generated samples start with ``\texttt{Q: A bakery makes}" and $1.25\%$ of generated samples start with ``\texttt{Q: A bakery makes cupcakes}". This could be explained as the consequence of the common pattern of the few-shot examples from the same training set, and memorized knowledge in the model. Because the iteration 0 model also shows similar phenomenon.

Furthermore, this phenomenon provides strong empirical support for Assumption~\ref{assumption} (ICL Capacity), which posits that the model performs Bayesian inference to identify the target domain distribution from few-shot examples. The frequent occurrence of common prefixes reflects the model’s robust extraction of the shared prior knowledge and structural patterns inherent in the provided examples $Z_{ctx}$. Essentially, the model effectively "locks on" to the specific domain definition (the common prefix) dictated by the prompt ($p(\hat{z}\mid\theta)$). Crucially, the subsequent divergence of the sequences confirms that the model is not simply memorizing a single optimal path (mode collapse), but is instead sampling diverse trajectories related to the context-specific distribution ($\prod_{i=1}^nf_{z^i}(\hat{z})$).
once the context is established. This duality—rigid adherence to the domain pattern (prefix) coupled with flexible generation (suffix)—validates that ICL successfully serves as a mechanism for distribution recovery rather than mere imitation.

\end{document}